\documentclass[twoside]{article}

\usepackage[accepted]{aistats2020}
\usepackage[noend]{algpseudocode}
\usepackage{amsmath}
\usepackage{amssymb}
\usepackage{amsfonts}
\usepackage{fullpage}
\usepackage{float}
\usepackage{tikz}
\usetikzlibrary{bayesnet}
\usetikzlibrary{arrows}
\usepackage{color}
\usepackage{xcolor}
\usepackage{graphicx}
\usepackage{caption}
\usepackage{subcaption}
\usepackage{amsthm}
\usepackage{natbib}
\usepackage{bm}
\usepackage{natbib}
\usepackage{mathtools}
\usepackage{url}
\usepackage{tikzpagenodes}
\usepackage [english]{babel}
\usepackage [autostyle, english = american]{csquotes}
\MakeOuterQuote{"}

\usetikzlibrary{backgrounds}
\newcommand{\bw}{\mathbf{w}}
\newcommand{\by}{\mathbf{y}}
\newcommand{\bz}{\mathbf{z}}
\newcommand{\bxi}{\bm{\xi}}
\newcommand{\btheta}{\bm{\theta}}
\newtheorem{thm}{Theorem}
\definecolor{lasallegreen}{rgb}{0.03, 0.47, 0.19}
\definecolor{jonquil}{rgb}{0.98, 0.85, 0.37}

\begin{document}

\twocolumn[

\aistatstitle{Prediction Focused Topic Models via Feature Selection}

\aistatsauthor{ Jason Ren* \And Russell Kunes* \And  Finale Doshi-Velez}
\aistatsaddress{jason\_ren@college.harvard.edu \And rk3064@columbia.edu \And finale@seas.harvard.edu}]

\begin{abstract}
  
Supervised topic models are often sought to balance prediction quality and interpretability.  However, when models are (inevitably) misspecified, standard approaches rarely deliver on both.  We introduce a novel approach, the prediction-focused topic model, that uses the supervisory signal to retain only vocabulary terms that improve, or at least do not hinder, prediction performance.  By removing terms with irrelevant signal, the topic model is able to learn task-relevant, coherent topics.  We demonstrate on several data sets that compared to existing approaches, prediction-focused topic models learn much more coherent topics while maintaining competitive predictions.

\end{abstract}

 \begin{tikzpicture}[remember picture,overlay]

    \node[align=center] at ([xshift=6.5em,yshift=1em]current page text area.south) 
    {*Equal Contribution };
  \end{tikzpicture}%

\section{Introduction}

Topic models such as Latent Dirichlet Allocation (LDA) \citep{Blei2003LatentDA} learn a small set of topics that capture the co-occurrence of terms in discrete count data.  Given the topics, any datum can be represented as a mixture over topics; these mixture weights can then be used as input for some down-stream supervised task. By reducing the dimensionality of the data, topic models help make predictions more interpretable, especially if the topics are themselves interpretable. This makes topic models valuable in contexts such as health care (e.g. \citet{hughes2017predictionanti}) or criminal justice (e.g. \citet{Kuang2017CrimeTM}) where interpretability is important for downstream validation.

However, learning topics that are both interpretable and predictive is tricky, and there generally exists a trade-off between the ability of the topics to predict the target well and explain the count data well. In real world settings in which the data do not truly come from a topic modeling generative process, the signal captured by a standard topic model may not be predictive of the target.  Supervised topic models attempt to mitigate this issue by including the targets during training, but the effect of including the targets is often minimal due to a cardinality mismatch between the data and the targets (\citep{zhang2014supervise}, \citep{halpern2012comparison}).

In this work, we focus on one common reason why topic models fail: documents often contain common, correlated terms that are irrelevant to the task.  For example, a topic model trained on movie reviews may form a topic that assigns high probability to words like "comedy", "action", "character", and "plot," which may be nearly orthogonal to a sentiment label. The existence of features irrelevant to the supervised task complicates optimization of the trade-off between prediction quality and explaining the count data, and also renders the topics less coherent.

To address this issue, we introduce a novel supervised topic model, prediction-focused sLDA (pf-sLDA), that explicitly allows features irrelevant to predicting the response variable to be ignored when learning topics. During training, pf-sLDA simultaneously learns which features are irrelevant and fits a supervised topic model on the relevant features.  It modulates the trade-off between prediction quality and explaining the count data with an interpretable model parameter. We demonstrate that compared to existing approaches, pf-sLDA is able to learn more coherent topics while maintaining competitive prediction accuracy on several data sets.

\section{Related Work}
\label{sec:related_work}

There are two main lines of related work: work to improve prediction quality in supervised topic models, and work to focus topics toward areas of interest.

\textbf{Improving prediction quality.}
Since the original supervised LDA (sLDA) work of \citet{mcauliffe2008supervised}, many works have incorporated information about the prediction target into the topic model training process in different ways to improve prediction quality, including power-sLDA \citep{zhang2014supervise}, med-LDA \citep{zhu2012medlda}, and BP-sLDA \citep{chen2015end}. 

\citet{hughes2017prediction} pointed out a number of shortcomings of these previous methods and introduced a new objective that weighted a combination of the conditional target likelihood and marginal data likelihood: $\lambda \log p(\by|\bw) + \log p(\bw)$. They demonstrated the resulting method, termed prediction-constrained sLDA (pc-sLDA),  achieves better empirical results in optimizing the trade-off between prediction quality and explaining the count data and justify why this is the case.  However, irrelevant but common terms are not handled well in their formulation; their topics are often polluted by irrelevant terms. Our pf-sLDA objective enjoys analogous theoretical properties but effectively removes irrelevant terms, and thus achieves more coherent topics. 

There are some supervised topic modeling approaches that model the existence of some sort of irrelevance in the data.  NUF-sLDA \citep{zhang2014supervise} models entire topics as noise. Topic models that encourage sparsity in topic vocabulary distributions, such as dual sparse topic models \citep{lin2014dual}, reduce the number of features per topic but do not easily allow an irrelevant feature to be ignored across all topics. To our knowledge, our approach is the first supervised topic modeling approach in which the model is able to learn which features are irrelevant to the supervised task and ignore them when fitting the topics.

\textbf{Focused topics.} 
The notion of focusing topics in relevant directions is also present in the unsupervised topic modeling literature.  For example, \citet{Wang2016TargetedTM} focus topics by seeding them with keywords; \citet{Kim2012VariableSF} introduce variable selection for LDA, which models some of the vocabulary as irrelevant.  \cite{Fan2017PriorMS} similarly develop stop-word exclusion schemes. However, these approaches adjust topics based on some general notions of "focused," whereas the goal of our pf-sLDA is to remove terms to explicitly manage a trade-off between prediction quality and explaining the count data.

\section{Background and Notation}

We briefly review supervised Latent Dirichlet Allocation (sLDA) \citep{mcauliffe2008supervised}. sLDA models count data (words) as coming from a mixture of $K$ topics $\{\beta_k\}_{k=1}^K$, where each topic $\beta_k \in \Delta^{|V|-1}$ is a categorical distribution over a vocabulary $V$ of $|V|$ discrete words. The count data are represented as a collection of $M$ documents, with each document $\bw_d \in \mathbb{N}^{|V|}$ being a vector of counts over the vocabulary. Each document $\bw_d$ is associated with a target $y_d$. Additionally, each document has an associated topic distribution $\theta_d \in \Delta^{K-1}$, which generates both the words and the target. 

The generative process is:
\begin{algorithmic}
\For{each document}
\State Draw topic distribution $\theta \sim \text{Dir}(\alpha)$
\For{each word}
\State Draw topic $z \sim \text{Cat}(\theta)$ 
\State Draw word $w \sim \text{Cat}(\beta_z)$
\EndFor
\State Draw target $y \sim \text{GLM}(\theta; \eta, \delta)$
\EndFor
\end{algorithmic}

\begin{table}[t]
    \centering
    \begin{tabular}{|c|c|}
          \hline
          Symbol & Meaning  \\ 
          \hline
          $K$ & Number of topics\\
          $M$ & Number of documents\\
          $V$ & Vocabulary \\
          $\bw$ &  Documents (bag of words) \\
          $\by$ & Targets \\ 
          $\beta$ & Topics \\
          $\btheta$ & Topic-document distributions\\
          $\alpha$ & Topic prior parameter \\
          $\eta, \delta$ & Target GLM parameters \\
          $\bxi$ & pf-sLDA channel switches  \\
          $p$ & pf-sLDA word inclusion prior \\ 
          $\pi$ & pf-sLDA additional topic  \\
          \hline
    \end{tabular}
    \caption{Notation}
    \label{tab:my_label}
\end{table}

\section{Prediction Focused Topic Models}
\label{sec:model}
\begin{figure*}[h!]
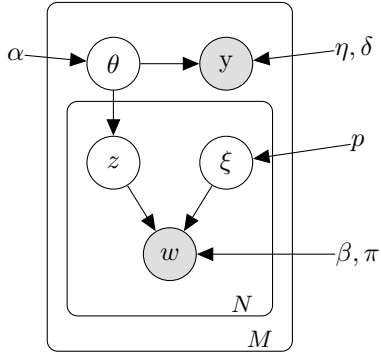

\begin{subfigure}{.5\textwidth}
  \centering
 \tikz{
% nodes
 \node[obs] (w) {$w$};%
 \node[latent,above=of w,xshift=-0.75cm,fill,yshift=-0.5cm] (z) {$z$}; %
 \node[latent,above=of w,xshift=0.75cm,yshift=-0.5cm] (xi) {$\xi$}; %
 \node[latent, above=of z, yshift=-0.4cm](theta){$\theta$};
 \node[const, above=of z,xshift = -1.3cm](alpha){$\alpha$};
 \node[obs, above= of xi, yshift=-0.4cm](y){y};
 \node[const, above=of xi,xshift = 1.7cm](ed){$\eta, \delta$};
 \node[const, xshift = 2.5cm] (betapi) {$\beta, \pi$};
 \node[const, above = of w, xshift=2.5cm](p) {$p$};
%  \node[latent, xshift = -2.5cm] (c) {$c$};
%  \node[const,  above=of c](rho) {$\rho$};
% plate
 \plate [inner sep=.25cm,yshift=.2cm] {plate1} {(z)(xi)(w)} {$N$}; %
 \plate [inner sep = .25cm,yshift=.2cm]
 {plate2} {(theta)(z)(xi)(w)(plate1)} {$M$};
%  \plate [inner sep = .25cm,yshift=.2cm]
%  {plate4} {(c)} {$V$}
% edgeshttps://www.overleaf.com/project/5d7be978e4aec500013e8dcf
 \edge {z,xi,betapi} {w} 
 \edge {p}{xi}
 \edge {theta}{z,y}
 \edge {alpha}{theta}
 \edge {ed}{y}
 }

  \label{fig:sub1}
\end{subfigure}%
\begin{subfigure}{.5\textwidth}
  \centering
  \begin{algorithmic}
\For{each document}
\State Draw topic distribution $\theta \sim Dir(\alpha)$
\For{each word}
\State Draw topic $z \sim Cat(\theta)$ 
\State Draw switch $\xi \sim Bern(p)$
\If{$\xi = 1$}
\State Draw word $w \sim Cat(\beta_z)$
\ElsIf{$\xi = 0$}
\State Draw word $w \sim Cat(\pi)$
\EndIf
\EndFor
\State Draw target $y \sim GLM(\theta; \eta, \delta)$ 
\State (See Appendix \ref{sec:impl_deets} for details)
\EndFor
\end{algorithmic}
  \label{fig:sub2}
\end{subfigure}
\caption{\textit{Left:} pf-sLDA graphical model. \textit{Right:} pf-sLDA generative process.}
\label{fig:graphical_models}
\end{figure*}

We now introduce prediction-focused sLDA (pf-sLDA), a novel, dual-channel topic model. The fundamental assumption that pf-sLDA builds on is that the vocabulary $V$ can be divided into two disjoint components, one of which is irrelevant to predicting the target variable.  pf-sLDA separates out the words irrelevant to predicting the target, even if they have latent structure, so that the topics can focus on only modeling structure that is relevant to predicting the target. The goal is to learn interpretable, focused topics that result in high prediction quality.

\paragraph{Generative Model.}

The pf-sLDA latent variable model has the following components: one channel of pf-sLDA models the count data as coming from a mixture of $K$ topics $\{\beta_k\}_{k=1}^K$, identical to sLDA. The second channel of pf-sLDA models the data as coming from an additional topic $\pi \in \Delta^{|V|-1}$.  (This modeling choice of a simple second channel is explained later in this section.) The target only depends on the first channel, so the second channel acts as an outlet for words irrelevant to predicting the target. We constrain $\beta$ and $\pi$ such that $\beta_k^\top \pi = 0$ for all $k$, such that each word is always either relevant or irrelevant to predicting the target. Which channel a word comes from is determined by its corresponding Bernoulli switch $\xi$, which has prior $p$. The generative process of pf-sLDA is given in Figure \ref{fig:graphical_models}. In Appendix \ref{sec:lik_lb_deriv}, we prove that a lower bound to the pf-sLDA likelihood is:
\begin{multline}
\log p(\by,\bw) \geq E_{\bxi}[\log p_{\beta}(\by|\bw, \bxi)] + \\ pE_{\btheta}[\log p_{\beta}(\bw|\btheta)] + (1-p)\log p_{\pi}(\bw)
\label{eqn:pfslda_lb}
\end{multline}
Here, the notation emphasizes the fact that the conditional pmf $p_{\beta}$ is free of $\pi$ and the multinomial pmf $p_{\pi}$ is free of $\beta$.

\textbf{Connection to prediction-constrained models.}
 The lower bound above reveals a connection to the pc-sLDA loss function of \citet{hughes2017prediction}. The first two terms capture the trade-off between performing the prediction task $E_{\bxi}[\log p_{\beta}(\by|\bw, \bxi)]$ and explaining the words $pE_{\btheta}[\log p_{\beta}(\bw|\btheta)]$, where the word inclusion prior $p$ is used to down-weight explaining the words (or emphasize the prediction task). This is analogous to the prediction-constrained objective, but we manage the trade-off through an interpretable model parameter, the word inclusion prior $p$, rather than a more arbitrary Lagrange multiplier $\lambda$.  Additionally, unlike pc-sLDA, our formulation is still a valid graphical model, and thus we can perform inference using standard Bayesian techniques. We describe similar connections in the true likelihood, rather than the lower bound, in Appendix \ref{sec:pf-sLDA_true_lik_to_pc}.

\textbf{Choosing a simple second channel.}
 The second term and third term capture the trade-off between how well a word can be explained by the relevant topics $\beta$ versus the additional topic $\pi$, weighted by $p$ and $1-p$ respectively.  This trade-off motivates why it is important to keep the additional topic $\pi$ simple (e.g. a single topic). If we fix a simple form for $\pi$, then $\pi$ will not be able to explain words as well as $\beta$. Thus, there is an inherent cost in considering words irrelevant (i.e. have mass in $\pi$ and not in $\beta$). Words will be considered irrelevant only if including them in the relevant topics $\beta$ hinders the prediction of the target $y$ in a way that outweighs the cost of considering them irrelevant given $p$.  Decreasing $p$ allows the model to emphasize prediction quality, but it also lowers the cost of considering a word irrelevant due to the weighting terms. A more expressive $\pi$ would lower the cost of considering words irrelevant even more and causes the model to put even relevant words into $\pi$.
 
All this together informs how pf-sLDA behaves. The word inclusion prior $p$ and the additional topic $\pi$ have specific interpretations based on generative process: the word inclusion prior $p$ is the prior probability that a word is relevant (has mass in $\beta$), while the additional topic $\pi$ contains the probabilities of each irrelevant word.  However, in the context of the loss, they play the role of governing the trade-off between prediction quality and explaining the count data, \emph{while still keeping the model a standard generative model}. As we decrease $p$, we encourage fewer words to be considered relevant, with the goal of improving prediction quality by ignoring irrelevant words. This allows pf-sLDA to improve prediction quality while maintaining the ability to model relevant words well. We can then tune $p$ with the goal of encouraging the model to consider the truly irrelevant words to be irrelevant without losing relevant signal. (See Figure~\ref{fig:main_res} in the results for a demonstration of this effect.)

\section{Inference}
\label{sec:inference}
Inference in the pf-sLDA framework corresponds to inference in a graphical model, so advances in Bayesian inference can be applied to solve the inference problem. In this work, we take a variational approach.  Our objective is to maximize the evidence lower bound (ELBO), with the constraint that the relevant topics $\beta$ and additional topics $\pi$ have disjoint support. The key difficulty is that of optimizing over the non-convex set $\{\beta, \pi : \beta^\top \pi = \mathbf{0}\}$. We resolve this with a strategic choice of variational family, which results in a straightforward training procedure that does not require any tuning parameters. In section \ref{sec:variational_family}, we discuss model properties that inform our decision of the variational family. In section \ref{sec:elbo}, we examine how this choice of variational family and optimization enforce our desired constraint.

\subsection{Properties and Variational Family}
\label{sec:variational_family}
We first note some properties of pf-sLDA that will help us choose our variational family.
\begin{thm}
If the channel switches $\bxi_d$ and the document topic distribution $\theta_d$ are conditionally independent in the posterior for all documents, then $\beta$ and $\pi$ have disjoint supports over the vocabulary.
\end{thm}
\begin{thm}
$\beta^\top \pi  = \bold{0}$ if and only if there exists an assignment to all channel switches $\bxi^*$ s.t.  $p(\bxi^* |\bw, \by) = 1$ 
\end{thm}
See Appendix \ref{sec:thm_proofs} for proofs.

Theorems 1 and 2 tell us that if the posterior distribution of the channel switches $\bxi_d$ is independent of the posterior distribution  of the document topic distribution $\theta_d$ for all documents, then the true relevant topics $\beta$ and additional topic $\pi$ must have disjoint support, and moreover the posterior of the channel switches $\bxi$ is a point mass. This suggests that to enforce that our generative model constraint that $\beta$ and $\pi$ are disjoint, we should choose the variational family such that $\bxi$ and $\btheta$ are independent. 

If $\bxi$ and $\btheta$ are conditionally independent in the posterior, then the posterior can factor as $p(\bxi, \btheta|\by, \bw) = p(\bxi |\by,\bw) p(\btheta | \by, \bw)$. In this case, the posterior for the channel switch of the $n$th word in document $d$, $\xi_{dn}$, has no dependence $d$, which can be seen directly from the graphical model. Thus, choosing $q(\bxi | \varphi)$ to have no dependence on document naturally pushes our assumption into the variational posterior. We specify the full form of our variational family below:
\begin{align*}
    q(\btheta, \bz, \bxi| \phi, \varphi, \gamma) &=  \prod_d q(\theta_d | \gamma_d) \prod_n q(\xi_{dn} |\varphi) q (z_{dn} | \phi_{dn}) \\
    \theta_d | \gamma_d &\sim \text{Dir}(\gamma_d) \\
     z_{dn} | \phi_{dn} &\sim \text{Cat}(\phi_{dn}) \\
     \xi_{dn} | \varphi &\sim \text{Bern}(\varphi_{w_{dn}})
\end{align*}
where $d$ indexes over the documents and $n$ indexes over the words in each document. Our choices for the variational distributions for $\btheta$ and $\bz$ match those of \citet{mcauliffe2008supervised}. We choose $q(\xi_{dn}|\varphi_{w_{dn}})$ to be a Bernoulli probability mass function with parameter $\varphi_{w_{dn}}$ indexed only by the word $w_{dn}$. This distribution acts as a relaxation of a true point mass posterior, allowing us to use gradient information to optimize over $[0,1]$ rather than directly over $\{0,1\}$.  Moreover, this parameterization allows us to naturally use the variational parameter $\varphi$ as a feature selector; low estimated values of $\varphi$ indicate irrelevant words, while high values of $\varphi$ indicate relevant words. We refer to $\varphi$ as our variational feature selector.

\subsection{ELBO and Optimization}
\label{sec:elbo}
Inference in the pf-sLDA corresponds to optimizing the following lower bound: 
\begin{align*}
    \text{ELBO}(\Theta, \Lambda) &= \log p_\Theta(\mathbf{y}, \bw) - \text{KL}(q_\Lambda(\zeta) || p_\Theta(\zeta|\bw, \mathbf{y}))
\end{align*}
where we denote the full set of model parameters as $\Theta$, the full set of variational parameters as $\Lambda$, and our latent variables as $\zeta \coloneqq \{\bxi, \btheta, \bz\}$.

We show how optimization of the ELBO, combined with our choice of variational family, will push $\beta$ and $\pi$ to be disjoint. First, we note that maximizing the ELBO over $\Theta$, $\Lambda$ is equivalent to: 
\begin{align*}
     \max_\Theta \left( \log p_\Theta(\mathbf{y}, \bw) - \min_\Lambda  \text{KL}(q_\Lambda(\zeta) || p_\Theta(\zeta|\bw, \mathbf{y})) \right)
\end{align*}
Therefore, estimating parameters based on the ELBO can be viewed as $\it{penalized}$ maximimum likelihood estimation with penalty: 
$$\text{pen}(\Theta) = \min_\Lambda \text{KL}(q_\Lambda(\zeta) || p_\Theta(\zeta|\bw, \mathbf{y}))$$

In the pf-sLDA model, the KL penalty for the channel switches $\bxi$ is minimized for the subset of the model parameter space $\{\beta, \pi : \beta^\top \pi = \bold{0}\}$. On this set, $\min_\Lambda \text{KL}(q_\varphi(\bxi) || p_\Theta(\bxi|\bw, \mathbf{y})) = 0 $ so long as $q_\varphi(\bxi)$ includes the set of point masses, as shown by Theorems $1$ and $2$. The analogous penalty everywhere else in the parameter space is greater than zero, based on the contrapositive of Theorem 1. We note that the size of this penalty and whether it is strong enough to enforce a disjoint solution for $\{\beta,\pi\}$ has to do with how restrictive the choice of variational family $q_\Lambda(\zeta)$ is. In our case, we found that our parameterization of $q_\Lambda(\zeta)$ resulted in estimates with $\beta$ and $\pi$ disjoint as desired.

Finally, to optimize the ELBO, we simply run mini-batch gradient descent on the ELBO (full form specified in Appendix \ref{sec:full_elbo}). We also explored using more traditional closed form coordinate ascent updates (see Appendix \ref{sec:caupdates}), but we found the gradient descent approach was computationally much faster and had very similar quality results.

\subsection{Prediction for new documents}
At test time, we make predictions given the learned GLM parameters and the MAP estimate of the topic distribution $\theta$ for the new document (estimated through gradient descent on the posterior for $\theta$).

% the posterior predictive mean: $E\left[y^* | \bw^*, \bw, \by; \Theta\right]$, where $y^*$ is the response for a new document $\bw^*$ and $\Theta$ has been fit on a training set. We follow the same procedure as sLDA. We first fit $q(\theta, \bz, \bxi)$, and then use it to approximate the posterior distribution. 
% \begin{align*}
%     E[y^* | \bw^*, \bw, \by, \Theta] \approx E_q[\mu(\theta)]
% \end{align*}
% Here, $\mu(\theta)$ is the mean of the GLM that $y$ is drawn from. As a practical matter, the vast majority of estimates of $\varphi_v$ converge to $\approx0$ or $\approx1$. Therefore, for a new document, we can condition on $\xi = \hat{\xi}_{MAP}$, and compute the posterior mean $E_q[\mu(\theta) | \xi = \hat{\xi}_{\text{MAP}}]$. It is easily shown that this conditional expectation does not depend on words in the vocabulary $v$ that have $\varphi_v = 0$. Thus, once the model has been fit, at test time we can drop these words from the data set, saving computation and memory and simplifying the prediction problem. 

\begin{figure*}[t]
\begin{subfigure}{.5\textwidth}
  \centering
  \begin{tabular}{|p{2cm}||p{1.5cm}|p{1.5cm}|}
 \hline
\multicolumn{3}{|c|}{Pang and Lee Movie Reviews} \\
 \hline
  Model & Coherence & RMSE  \\

\hline
 sLDA & 0.362 (0.101) & 1.682 (0.021) \\
 \hline
pc-sLDA & 0.492 (0.130) & \textbf{1.298} (0.015) \\
 \hline 
 pf-sLDA & \textbf{1.987} (0.092) & 1.305 (0.024) \\
 \hline 
 \hline
 \multicolumn{3}{|c|}{Yelp Reviews} \\
 \hline
 Model & Coherence & RMSE  \\

\hline
  sLDA & 0.848 (0.086) & 1.162 (0.017) \\
 \hline
pc-sLDA & 1.080 (0.213) & 0.953 (0.004)\\
 \hline 
 pf-sLDA & \textbf{3.258} (0.102) & \textbf{0.952} (0.011) \\
 \hline 
 \hline
  \multicolumn{3}{|c|}{ASD Dataset} \\
  \hline
  Model & Coherence & AUC  \\
\hline
 sLDA & 1.412 (0.113) & 0.590 (0.013)\\
 \hline
pc-sLDA & 2.178 (0.141)  & 0.701 (0.015) \\
\hline
pf-sLDA & \textbf{2.639} (0.091) & \textbf{0.748}  (0.013)\\
 \hline 

\end{tabular}

  \label{fig:sub1}
\end{subfigure}
\begin{subfigure}{.5\textwidth}
  \includegraphics[width=78mm]{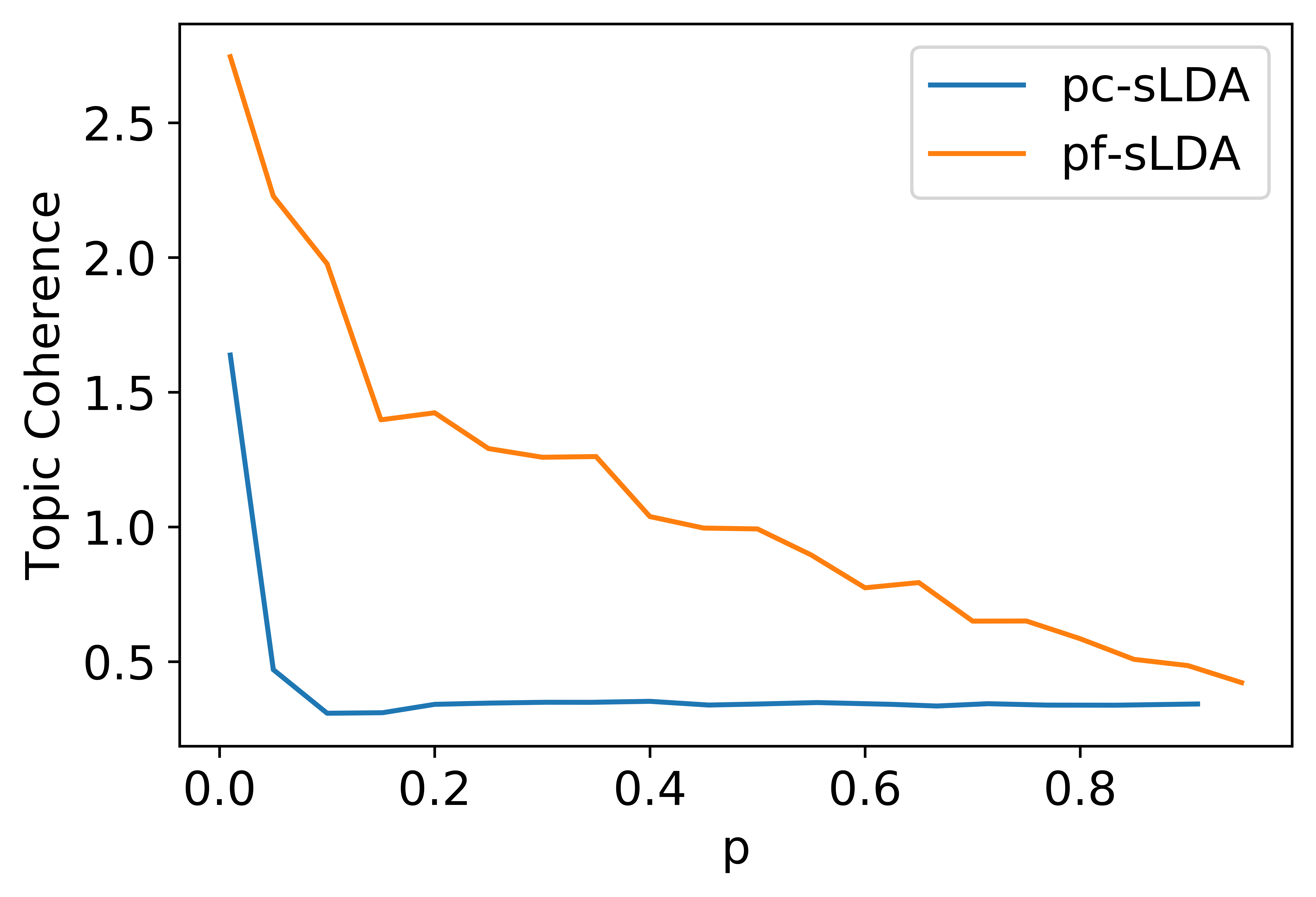}
  \includegraphics[width=78mm]{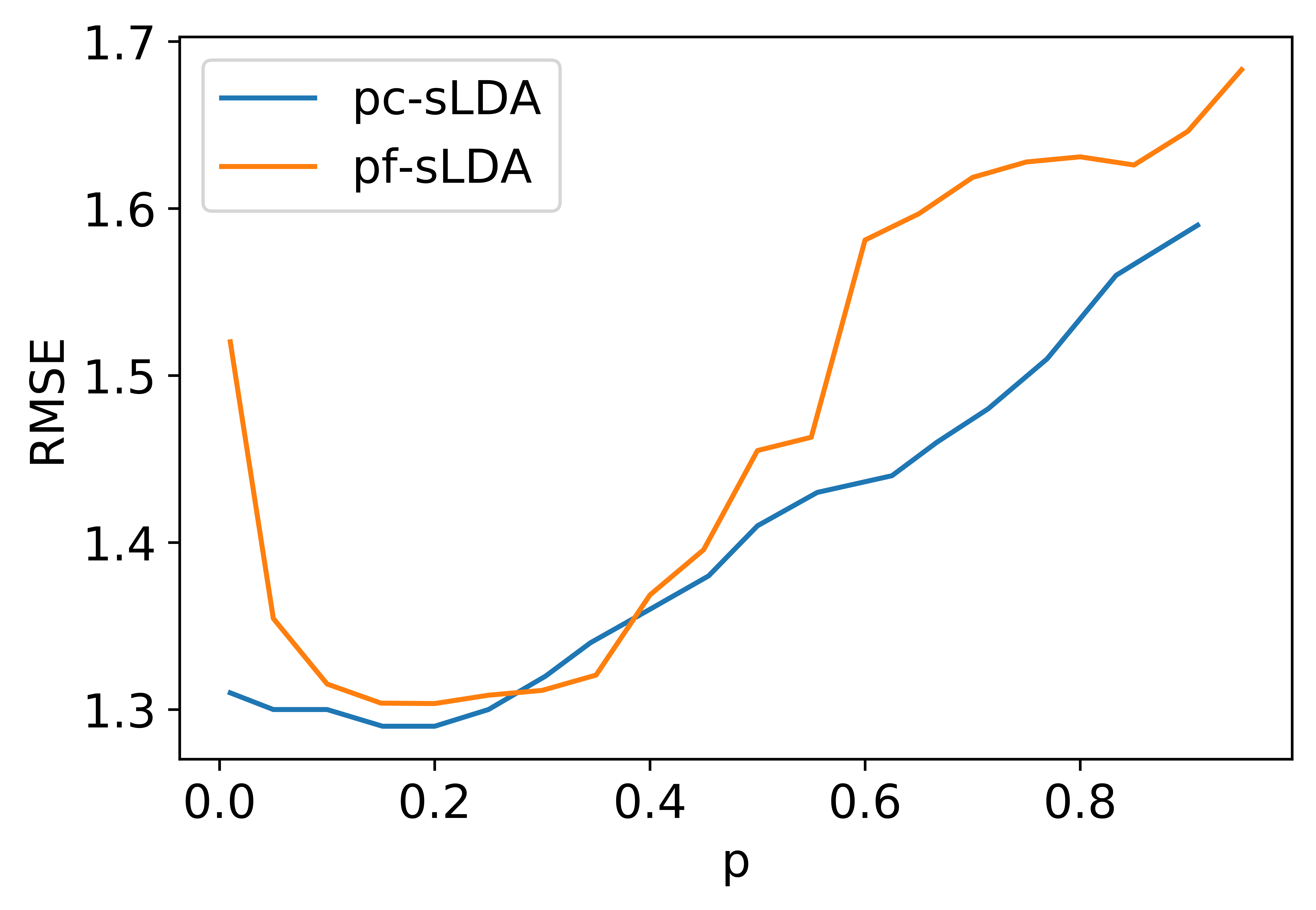}
% Uncomment pictures
\end{subfigure}

\caption{\textit{Left:} Mean and (SD) across 5 runs for topic coherence (higher is better) and RMSE (lower is better) or AUC (higher is better) on held-out test sets. Final models were chosen based on a combination of validation coherence and RMSE/AUC. pf-sLDA produces topics with much higher coherence across all three data sets, while maintaining similar prediction performance.  \textit{Right:} Validation coherence and RMSE on Pang and Lee Movie Review data set as $p = \frac 1 \lambda$ varies. This demonstrates the effect of the channel word inclusion prior $p$ controlling the trade-off between prediction quality and explaining the count data.}
\label{fig:main_res}
\end{figure*}

\begin{table*}[h!]
\centering
\begin{tabular}{|p{1cm}||p{4.5cm}|p{4.5cm}|p{4.5cm}|}
 \hline
\multicolumn{4}{|c|}{Pang and Lee Movie Reviews} \\
 \hline
     & sLDA & pc-sLDA & pf-sLDA\\
 \hline
 High & motion, way, love, performance, \textcolor{lasallegreen}{best}, picture, films, 
character, characters, life &
\textcolor{lasallegreen}{best}, little, time, \textcolor{lasallegreen}{good}, don, picture, year, rated, films
just & 
\textcolor{lasallegreen}{brilliant}, \textcolor{lasallegreen}{rare}, \textcolor{lasallegreen}{perfectly}, true, \textcolor{lasallegreen}{oscar}, documentary, \textcolor{lasallegreen}{wonderful}, 
\textcolor{lasallegreen}{fascinating}, \textcolor{lasallegreen}{perfect}, \textcolor{lasallegreen}{best} \\
%  $\eta$ &  7.801 & 8.287 & 10.063  \\
\hline \hline

Low & plot, time, \textcolor{red}{bad},  \textcolor{lasallegreen}{funny},  \textcolor{lasallegreen}{good},  \textcolor{lasallegreen}{humor}, little, isn, action & script, year, little,  \textcolor{lasallegreen}{good}, don, look, rated, picture, just, films 
& \textcolor{red}{awful}, \textcolor{red}{stupid}, gags, \textcolor{red}{dumb}, \textcolor{red}{dull}, sequel, \textcolor{red}{flat}, \textcolor{red}{worse}, \textcolor{red}{ridiculous}, \textcolor{red}{bad} \\
%  $\eta$ &  3.516 & 2.802 & 0.122 \\
\hline \hline

\multicolumn{4}{|c|}{Yelp Reviews} \\
 \hline
     & sLDA & pc-sLDA & pf-sLDA \\
 \hline
 High &fries, \textcolor{lasallegreen}{fresh}, burger, try, cheese, really, pizza, place, \textcolor{lasallegreen}{like}, \textcolor{lasallegreen}{good} &
\textcolor{lasallegreen}{best}, just, \textcolor{lasallegreen}{amazing}, \textcolor{lasallegreen}{love}, \textcolor{lasallegreen}{good}, food, service, place, time, \textcolor{lasallegreen}{great} 
& 
\textcolor{lasallegreen}{fantastic}, \textcolor{lasallegreen}{loved}, highly, \textcolor{lasallegreen}{fun}, \textcolor{lasallegreen}{excellent}, \textcolor{lasallegreen}{awesome}, atmosphere, \textcolor{lasallegreen}{amazing}, \textcolor{lasallegreen}{delicious}, \textcolor{lasallegreen}{great} \\

%  $\eta$ &  5.392 & FILL & 60183  \\
\hline \hline

Low & store, time, want, going, place, know, people, don, \textcolor{lasallegreen}{like}, just
& 
didn, don, said, told, like, place, time, just, service, food 
&  \textcolor{red}{awful}, management,  \textcolor{red}{dirty},  \textcolor{red}{poor},  \textcolor{red}{horrible},  \textcolor{red}{worst},  \textcolor{red}{rude},
 \textcolor{red}{terrible}, money,  \textcolor{red}{bad} \\

%  $\eta$ &  1.302 & FILL & -0.409 \\
\hline \hline
\multicolumn{4}{|c|}{ASD} \\
 \hline
     & sLDA & pc-sLDA & pf-sLDA \\
 \hline
High & Intellect disability  & Infantile cerebral palsy & \textcolor{lasallegreen}{Other convulsions} \\
& Esophageal reflux &  Congenital quadriplegia & Aphasia \\
& Hearing loss & Esophageal reflux & \textcolor{lasallegreen}{Convulsions} \\
& Development delay & fascia Muscle/ligament dis & Central hearing loss \\
& Downs syndrome & Feeding problem &  \textcolor{lasallegreen}{Grand mal status}\\
\hline 
Low & Otitis media & Accommodative esotropia & Autistic disorder \\
& Asthma & Joint pain-ankle & Diabetes Type 1 c0375114\\
& Downs syndrome & Congenital factor VIII & Other symbolic dysfunc \\
& Scoliosis & Fragile X syndrome & Diabetes Type 1 c0375116 \\
& Constipation& Pain in limb & Diabetes Type 2 \\
\hline

\end{tabular}

\caption{ We list the most probable words in the topics with the highest and lowest regression coefficient for each model and data set. In the context of each data set, words expected to be in a high-regression-coefficient topic are listed in green, and words expected to be in a low-regression-coefficient topic are listed in red. These colors are done by hand and are only meant for ease of evaluation. It is clear that the topics learned by pf-sLDA are the most coherent and contain the most words with task relevance.}

\label{tab:topics}
\end{table*}

\section{Experimental Results} 
In this section, we demonstrate the ability of pf-sLDA to find more coherent topics while maintaining competitive prediction quality when compared to pc-sLDA. We also demonstrate in the well-specified case, pf-sLDA is able to reliably and accurately recover the relevant features. Finally, we show the feature filtering of pf-sLDA outperforms naive feature filtering based on correlation to target.

\subsection{Experimental Set-Up}
\textbf{Metrics.}
We wish to assess prediction quality and interpretability of learned topics. To measure prediction quality, we use RMSE for real targets and AUC for binary targets. To measure interpretability of topics, we use normalized pointwise mutual information coherence, as it was found by \citet{Newman2010AutomaticEO} to be the that metric most consistently and closely matches human judgement in evaluating interpretability of topics. See Appendix \ref{sec:coherence_deets} for an explicit procedure on how we calculate coherence.

\textbf{Baselines.}
The recent work in \cite{hughes2017prediction} demonstrates that pc-sLDA outperforms other supervised topic modeling approaches, so we use pc-sLDA as our main baseline. We also include standard sLDA \citep{mcauliffe2008supervised} for reference.

\textbf{Data Sets}.
 We run our model and baselines on three real data sets:
 \begin{itemize}
    \item Pang and Lee's movie review data set \citep{Pang2005SeeingSE}: 5006 movie reviews, with integer ratings from $1$ (worst) to $10$ (best) as targets.
    \item Yelp business reviews \citep{Yelp:2019}: 10,000 business reviews, with integer stars from $1$ (worst) to $5$ (best) as targets.
    \item Electronic health records (EHR) data set of patients with Autism Spectrum Disorder (ASD), introduced in \cite{Masood2018APV}: 3804 EHRs, with binary indicator of epilepsy as target.
\end{itemize}
For more details about data sets, see Appendix \ref{sec:dataset_deets}.

\textbf{Implementation details}. Refer to Appendix \ref{sec:impl_deets}

\begin{figure*}[h!]
\begin{subfigure}{.55\textwidth}
\begin{tabular}{|p{3cm}||p{2cm}|p{2cm}|}
 \hline
\multicolumn{3}{|c|}{Pang and Lee Movie Reviews} \\
 \hline
  Filtering Method & Coherence & RMSE  \\
 \hline
 p = 10, pf-sLDA  & 1.916 (0.105) & 1.418 (0.024)\\
\hline
 p = 10, correlation & 1.115 (0.092) & 1.729 (0.031) \\
\hline
 p = 15, pf-sLDA & 1.509 (0.077) & 1.313 (0.008) \\
\hline
p = 15, correlation  & 0.875 (0.088) & 1.694 (0.017) \\
\hline
\end{tabular}
\end{subfigure}
\begin{subfigure}{.45\textwidth}
\includegraphics[width=70mm]{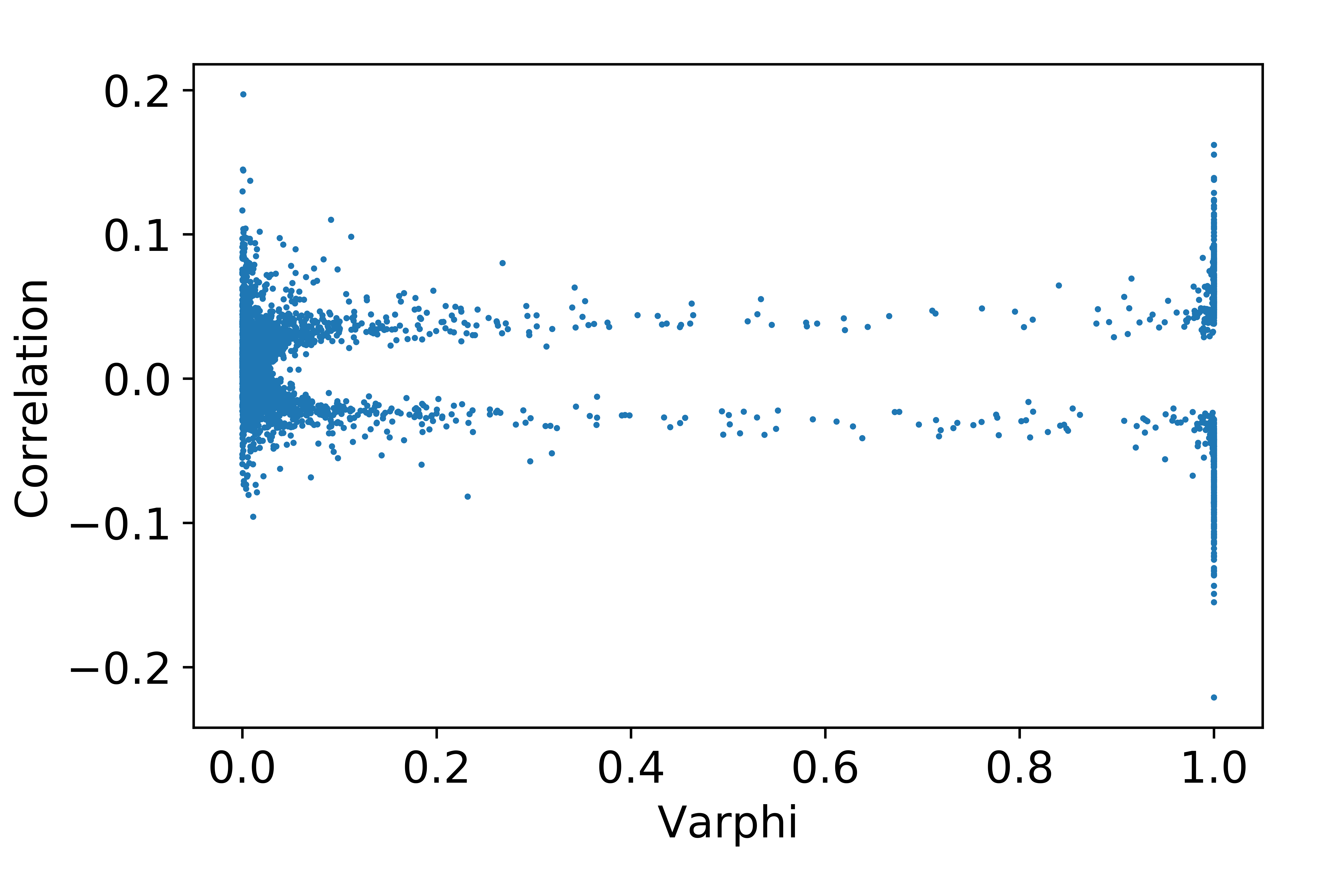}
\end{subfigure}
\caption{\textit{Left:} Mean and (SD) across 5 runs for topic coherence (higher is better) and RMSE (lower is better) for sLDA on a filtered vocabulary on a held-out test set. Filtering by pf-sLDA as compared to correlation results better coherence and RMSE. \textit{Right:} Each dot represents one word in the vocabulary. We plot the variational feature selector $\varphi$ after training vs. the correlation to the target for each word}
\label{fig:filter}
\end{figure*}

\begin{table}[h!]
\begin{center}
 \begin{tabular}{|p{1cm}||p{2.5cm}|p{2.5cm}|}
 \hline
\multicolumn{3}{|c|}{Synthetic Data Set, true p = 0.25} \\
 \hline
  p & Precision & Recall  \\

\hline
 0.15 & 1.000 (0.000) & 0.430 (0.013) \\
 \hline
 0.25 & 0.981 (0.006) & 0.962 (0.006) \\
 \hline 
 0.35 & 0.885 (0.008) & 0.982 (0.006) \\
 \hline 
 \hline
\end{tabular}
\end{center}
\caption{Mean and (SD) of precision and recall of relevant features being considered relevant by pf-sLDA with word inclusion prior $p$ on a synthetic data set. In the well specified case, pf-sLDA is able to recover the relevant features with high precision and recall.} 
\label{tab:synthetic}
\end{table}

\subsection{Results}
\textbf{pf-sLDA learns the most coherent topics.} 
Across data sets, pf-sLDA learns the most coherent topics by far (see Figure \ref{fig:main_res}). pc-sLDA improves on topic coherence compared to sLDA, but cannot match the performance of pf-sLDA. Qualitative examination of the topics in Table \ref{tab:topics} supports the claim that the pf-sLDA topics are more coherent, more interpretable, and more focused on the supervised task. 

Furthermore, if we filter the learned topics of either sLDA or pc-sLDA post hoc based on the variational feature selectors of a trained pf-sLDA model, the topics (and coherence scores) resemble the original pf-sLDA's. This supports the notion that pf-sLDA is focusing on relevant signal.

\textbf{Prediction quality of pf-sLDA remains competitive.}
pf-sLDA produces similar prediction quality compared to pc-sLDA across data sets (see Figure \ref{fig:main_res}). Both pc-sLDA and pf-sLDA outperform sLDA in prediction quality. In the best performing models of pf-sLDA for all 3 data sets, generally between $5\%$ and $20\%$ of the words were considered relevant. Considering both more words or less words relevant hurt performance, as seen in the plots in Figure \ref{fig:main_res}.

\textbf{In the well-specified case, pf-sLDA can recovers relevant features with high precision and recall.}
Our real world examples demonstrate that pf-sLDA achieves better topic coherency with similar prediction performance to state-of-the-art approaches in supervised topic modeling. To test how well pf-sLDA removes irrelevant features, we now turn to a simulated setting in which words are generated from the pf-sLDA generative process with $p=0.25$. (See details on simulation setting in Appendix \ref{sec:dataset_deets}.)

We run pf-sLDA with random initializations 10 times on each of the simulated data sets. We report mean and standard deviation of precision and recall of relevant features being labeled as relevant by pf-sLDA. We consider pf-sLDA labelling a feature as relevant if the corresponding $\varphi$ is greater than $0.99$.

Under the well specified case when we have the correct word inclusion prior $p$, we see that pf-sLDA is able to correctly identify which features are relevant with high precision and recall (see Table \ref{tab:synthetic}). We get results as expected if we misspecify $p$. Precision increases and recall decreases with lower $p$, since this encourages the model to believe fewer of the words are actually relevant. Precision decreases and recall increases with higher $p$, since this encourages the model to believe more of the words are actually relevant. Furthermore, the model is able to reproduce high precision and recall consistently across random initialization, as seen by the low standard deviations.

\textbf{The variable selection of pf-sLDA outperforms naive variable selection based on correlation to target.} Another question of interest was whether using pf-sLDA as a feature selector is any better than simply retaining words that correlate to the target.  To test this, we consider two different filtering schemes. The first scheme retains words that under pf-sLDA have their variational feature selector $\varphi_v > 0.99$.  The second scheme retains the top-$N$ highest correlated words to the target, where $N$ is chosen to be the same number as the words retained by pf-sLDA. We then run a standard sLDA on the filtered vocabulary. We run this experiment on Pang and Lee's movie review data set.

From Figure \ref{fig:filter}, we see that running sLDA on a filtered vocab based on pf-sLDA produces better results in terms of both topic coherence and RMSE than on a filtered vocab based on correlation. Additionally, it is clear pf-sLDA throws out words that have close to zero correlation with the target. For words that have some correlation, pf-sLDA identifies and retains those that are more useful for prediction.

 \section{Discussion}
The introduction of pf-sLDA was motivated by the observation that the presence of irrelevant features in real world data sets hampers the ability of topic models to recover topics that are simultaneously coherent and relevant for a supervised task. 

\textbf{Learning and Inference.}
From a learning perspective, one of the advantages of our approach compared to prediction constrained training is that pf-sLDA can be specified via a graphical model. Thus, we can enjoy the benefits of a principled loss and graphical model inference. For example, while we used SGD to optimize the ELBO, we could have easily incorporated classic updates for some of the parameters as in \cite{mcauliffe2008supervised}.  The most challenging aspect was maintaining orthogonality between the relevant topics $\beta$ and the additional topic $\pi$. In all the settings we tested, our parameterization of $q(\bxi|\bold{\varphi})$ that removes dependence on document enforced $\beta^\top \pi = \bold{0}$.  Proving properties of this choice and the approach of incorporating model constraints into the variational family and inference is an area for future work.

\textbf{Interpretable Predictions.}
Our experiments show that the topics learned by pf-sLDA are more coherent than other methods, using a metric that aligns closely with human judgement of interpetability (\cite{Newman2010AutomaticEO}). Moreover, qualitative examination of the topics and selected words for a variety of supervised tasks supports this conclusion: topics learned by pf-sLDA are highly coherent while maintaining task relevance.  For example, in a deeper qualitative analysis of the Pang and Lee movie review data set, where targets are integers from 1 to 10 (10 being most positive and 1 being most negative), pf-sLDA produces topics with corresponding regression coefficients close to 10 that are highly positive in sentiment, and topics with coefficients close to 0 that are highly negative in sentiment, while topics with coefficients close to 5 and 6 are more mild in sentiment (see Table \ref{tab:Pang_lee_full_topics} in Appendix \ref{sec:full_PL_topics}). This contrasts with the pc-sLDA and sLDA topics, in which topic sentiments are not particularly clear.

\textbf{Feature Selection and Applications.}  
The pf-sLDA approach is particularly well suited for tasks where it is suspected that there are a large number of features unrelated to a target variable of interest.  For example, beyond sentiment and epilepsy detection, one could imagine pf-sLDA being useful for tasks such as genome wide association studies where a majority of genetic data is irrelevant to a certain disease risk but there exists a latent structure describing a small set of relevant genes. 

In Figure \ref{fig:filter}, we demonstrated the effectiveness of the feature selection of pf-sLDA by showing retaining features based on the variational feature selectors ($\varphi$) of pf-sLDA outperforms simply retaining features with strong bivariate correlation with the target. Intuitively, pf-sLDA preserves co-occurrence relationships between the features that may be useful for a wide range of possible downstream tasks, which selection based solely on the bivariate correlations may ignore.

\section{Conclusion}
In this paper, we introduced prediction-focused supervised LDA, whose feature selection procedure improves both predictive accuracy and topic coherence of supervised topic models.  The model enjoys good theoretical properties, inferential properties, and performed well on real and synthetic examples. Future work could include establishing additional theoretical properties of the pf-sLDA variable selection procedure and applying our trick of managing trade-offs within a graphical model for variable selection in other generative models.

\section*{Acknowledgements}

FDV and JR acknowledge support from NSF CAREER 1750358 and Oracle Labs.  JR also acknowledges support from Harvard's PRISE program over Summer 2019.

\bibliography{refs} 
\bibliographystyle{ieeetr}
\newpage
\onecolumn
\section{Appendix}
\subsection{Code Base}
\url{https://github.com/jasonren12/PredictionFocusedTopicModel}

\setcounter{thm}{0}
\subsection{Theorem Proofs}
\label{sec:thm_proofs}
\begin{thm}
Suppose that the channel switches $\xi_d$ and the document topic distribution $\theta_d$ are conditionally independent in the posterior for all documents, then $\beta$ and $\pi$ have disjoint supports over the vocabulary.
\end{thm}

\begin{proof} For simplicity of notation, we assume a single document and hence drop the subscripts on $\xi_d$ and $\theta_d$. All of the arguments are the same in the multi-document case.
If $\xi$ and $\theta$ are conditionally independent in the posterior, then we can factor the posterior as follows: $p(\xi, \theta | \bw, y) = p(\xi | \bw, y) p(\theta | \bw, y)$. We expand out the posterior:
\begin{align*}
    p(\xi, \theta | \bw, y) &\propto p(\xi) p(\theta) p(\bw, y | \theta, \xi) \\
    &\propto p(\xi) p(\theta) p(y | \theta) \prod_n p_\beta(w_n | \theta) ^{\xi_n} p_\pi(w_n)^{1 - \xi_n} \\
    &= f(\theta) g(\xi) \prod_n p_\beta(w_n | \theta) ^{\xi_n}
\end{align*}

for some functions $f$ and $g$. Thus we see that we must have that $ \prod_n p_\beta(w_n | \theta) ^{\xi_n}$  factors into some $r(\theta) s(\xi)$. We expand $\prod_n p_\beta(w_n | \theta) ^{\xi_n}$:
\begin{align*}
    p_\beta(w_n | \theta)^{\xi_n} &= \left( \sum_k \beta_{k, w_n} \theta_k\right)^{\xi_n} \\
    &= I(\xi_n = 0) + I(\xi_n = 1) \left( \sum_k \beta_{k, w_n} \theta_k\right) 
\end{align*}
So that we can express the product as:
\begin{align*}
    \prod_n p_\beta(w_n | \theta)^{\xi_n} &= \prod_n \left\{ I(\xi_n = 0) + I(\xi_n = 1) \left( \sum_k \beta_{k, w_n} \theta_k\right)\right\}
\end{align*}
In order to further simplify, let $\beta_0 = \{n : \sum_k \beta_{k, w_n} = 0\}$ and $\beta_{>} = \{n : \sum_k \beta_{k, w_n} > 0\}$. In other words $\beta_0$ is the set of $n$ such that the word $w_n$ is not supported by $\beta$, and $\beta_{>}$ is the set of $n$ such that the word $w_n$ is supported by $\beta$.

We can rewrite the above as:
\begin{align*}
    \prod_n p & _\beta(w_n | \theta)^{\xi_n} = \left(\prod_{n \in \beta_0} I(\xi_n = 0) \right) \left(\prod_{n \in \beta_>} \left\{ I(\xi_n = 0) +  I(\xi_n = 1)  \sum_k \beta_{k, w_n} \theta_k \right\} \right)
\end{align*}

Thus, we see that we can factor $\prod_n p_\beta(w_n | \theta) ^{\xi_n}$ as a function of $\theta$ and $\xi$ into the form $r(\theta) s(\xi)$ only if $\xi_n = 0$ or $\xi_n =1$ with probability $1$. We can check that this implies $\beta_k^\top \pi = 0$ for each $k$ by the result of Theorem $2$. 
\end{proof}

\begin{thm}
$\beta^\top \pi  = 0$ if and only if there exists a $\bxi^*$ s.t.  $p(\bxi^* |\bw, y) = 1$ 
\end{thm}

\begin{proof}

\begin{enumerate}
    \item Assume $\beta^\top \pi  = 0$. Then, conditional on $w_n$, $\xi_n = 1$ with probability 1 if $\pi_{w_n} = 0$ and  $\xi_n = 0$ with probability 1 if $\pi_{w_n} > 0$. So we have $p(\bxi^* |\bw, y) = 1$ for the $\bxi^*$ corresponding to $\bw$ as described before.
    \item Assume there exists a $\bxi^*$ s.t.  $p(\bxi^* |\bw, y) = 1$.
    
    Then we have:
    \begin{align*}
        p(\bxi^* | \bw, y) = \frac{p(\bw, y | \bxi^*) p(\bxi^*)}{\sum_{\bxi} p(\bw, y | \bxi) p(\bxi)} = 1 \\
        p(\bw, y | \bxi^*) p(\bxi^*) = \sum_{\bxi} p(\bw, y | \bxi) p(\bxi)
    \end{align*}
    
    This implies $p(\bw, y | \bxi) p(\bxi) = 0 \ \forall\  \bxi \neq \bxi^*$, which implies $p(\bw, y | \bxi) = 0 \ \forall\ \bxi \neq \bxi^*$
    
    Then we have:
    \begin{align*}
        p(\bw, y | \bxi) &= p(y |\bw, \bxi) p(\bw| \bxi) \\ 
        &= \left(\int_\theta p(y | \theta) p(\theta | \bw, \bxi) d\theta \right)\left( \int_\theta p(\bw | \theta, \bxi) p(\theta) d\theta \right)
    \end{align*}
    
    The first term will be greater than 0 because $y|\theta$ is distributed Normal. We focus on the second term.
    
    \begin{align*}
         \int_\theta p(\bw | \theta, \bxi) p(\theta) d\theta = \int_\theta p(\theta) \prod_n p_\beta(w_n | \theta) ^{\xi_n} p_\pi(w_n)^{1 - \xi_n}d\theta
    \end{align*}
    
    Let $X$ be the set of $\bxi$ that differ from $\bxi^*$ in one and only one position, i.e. $\xi_n = \xi^*_n$ for all $n \in \{1,\dots N\}\setminus \{i\}$ and $\xi_i \neq \xi^*_i$. For each $\xi \in X$, $\int_\theta d\theta p(\theta) \prod_n p_\beta(w_n | \theta) ^{\xi_n} p_\pi(w_n)^{1 - \xi_n} = 0$. Since all functions in the integrand are non-negative and continuous, $p_\beta(w_n | \theta) ^{\xi_n} p_\pi(w_n)^{1 - \xi_n} = 0$ for all $\theta$ for the unique $i$ with $\xi_i \neq \xi_i^*$. Since this holds for every element of $X$, we must have that $p_\beta(w_n | \theta) = 0$ for all $\xi_n = 0$ and $p_\pi(w_n) = 0$ for all $\xi_n = 1$, proving $\beta$ and $\pi$ are disjoint, provided the minor assumption that all words in the vocabulary $w_n$ are observed in the data. In practice all words are observed in the vocabulary because we choose the vocabulary based on the training set.
\end{enumerate}
\end{proof}

\subsection{ELBO (per doc)}
\label{sec:full_elbo}
Let $\Lambda= \{ \alpha, \beta, \eta, \delta, \pi, p \} $. Omitting variational parameters for simplicity:
\begin{align*}
    \log p(\bw,\by | \Lambda) &= \log \int_\theta \sum_z \sum_\xi p(\theta, \bz, \bxi, \bw, \by | \Lambda) d\theta \\
    &= \log E_q \left(\frac{p(\theta, \bz, \bxi, \bw, \by | \Lambda)} {q(\theta, \bz, \bxi)}\right) \\
    &\geq E_q[\log p(\theta, \bz, \bxi, \bw, \by)] - E_q[q(\theta, \bz, \bxi)]
\end{align*}

Let $ELBO = E_q[\log p(\theta, \bz, \bxi, \bw, \by | \Lambda)] - E_q[q(\theta, \bz, \bxi)]$

Expanding this:
\begin{align*}
    ELBO &= E_q[\log p(\theta | \alpha)] + E_q[\log p(\bz|\theta)] + E_q[\log p(y|\theta, \eta, \delta)]\\
    &+ E_q[\log p(\bxi| p)] + E_q[\log p(\bw | \bz, \beta, \bxi, \pi)] \\
    &- E_q[\log q(\theta | \gamma)] - E_q[\log q(\bz | \phi)] - E_q[\log q(\bxi | \varphi)]
\end{align*}
The distributions of each of the variables under the generative model are:
\begin{align*}
    &\theta_d \sim \text{Dirichlet}(\alpha)\\
    &z_{dn} |\theta_d \sim \text{Categorical}(\theta_d)\\
    &\xi_{dn} \sim \text{Bernoulli}(p)\\
    &w_{dn}|z_{dn},\xi_{dn} = 1 \sim \text{categorical}(\beta_{z_{dn}})\\
    &w_{dn}|z_{dn},\xi_{dn} = 0 \sim \text{Categorical}(\pi)\\
    &y_d|\theta_d \sim \text{GLM}(\theta ; \eta, \delta)
\end{align*}
Under the variational posterior, we use the following distributions:
\begin{align*}
    &\theta_d \sim \text{Dirichlet}(\gamma_d)\\
    &z_{dn} \sim \text{Categorical}(\phi_{dn})\\
    &\xi_{dn} \sim \text{Bernoulli}(\varphi_{w_{dn}})
\end{align*}
    
This leads to the following ELBO terms:
\begin{align*}
    E_q[\log p(\theta | \alpha)] &=  \log \Gamma\left(\sum_k \alpha_k\right) - \sum_k \log \Gamma(\alpha_k) 
    + \sum_k(\alpha_k - 1)E_q[\log \theta_k] \\
    E_q[\log p(\bz | \theta)] &= \sum_n \sum_k \phi_{nk}E_q[\log \theta_k] \\
    E_q[\log p(\bw | \bz, \beta, \bxi, \pi)] &= \sum_n \left(\sum_v w_{nv} \varphi_v \right) *\left(\sum_k \sum_v \phi_{nk} w_{nv} \log \beta_{kv}\right) + \left(1 - \left(\sum_v w_{nv} \varphi_v \right)\right) \left(\sum_v w_{nv} \log \pi_v \right) \\
    E_q[\log p(\bxi| p)] &= \sum_n \left(\sum_v w_{nv} \varphi_v \right) \log p  + \left(1 - \left(\sum_v w_{nv} \varphi_v \right)\right)\log (1-p) \\
    E_q[q(\theta | \gamma)] &= \log \Gamma\left(\sum_k \gamma_k\right) - \sum_k \log \Gamma(\gamma_k) + \sum_k(\gamma_k - 1)E_q[\log \theta_k] \\
    E_q[q(\bz | \phi)] &= \sum_n \sum_k \phi_{nk} \log \phi_{nk} \\
    E_q[q(\bxi| \varphi)] &= \sum_n \left(\sum_v w_{nv} \varphi_v \right) \log \left(\sum_v w_{nv} \varphi_v \right) + \left(1 - \left(\sum_v w_{nv} \varphi_v \right)\right) \log \left(1 - \left(\sum_v w_{nv} \varphi_v \right)\right) \\
    E_q[\log p(y|\theta, \eta, \delta)] &= \frac{1}{2}\log 2\pi \delta - \frac{1}{2\delta}\left(y^2 -2y\eta^\top E_q[\theta] + \eta^\top E_q[\theta \theta^\top]\eta \right)
\end{align*}

Other useful terms:
\begin{align*}
E_q[\log \theta_k] &= \Psi(\gamma_k) - \Psi\left(\sum_{j=1}^K\gamma_{j}\right) \\
\bar{Z} &:= \frac {\sum_n \xi_n z_n} {\sum_n \xi_n} \in \mathbb{R}^K \\
E_q[\theta] &= \frac {\gamma} {\gamma^\top \bold{1}} \\
\gamma_0 &:= \sum_k \gamma_k \\
\tilde{\gamma}_j &:= \frac {\gamma_j} {\sum_k \gamma_k} \\
E_q[\theta \theta^\top]_{ij} &= \frac{\tilde{\gamma}_i (\delta(i,j) - \tilde{\gamma}_j)}{\gamma_0 + 1} + \tilde{\gamma}_i\tilde{\gamma}_j
\end{align*}

% \begin{align*}
% L(\gamma, \varphi, \phi; \alpha, \beta, p) = &\sum_{d=1}^M \Bigg\{ \log \Gamma(\sum_{i=1}^K \alpha_i) - \sum_{i=1}^K \log \Gamma(\alpha_i) + \sum_{i=1}^K (\alpha_i - 1)(\Psi(\gamma_{di}) - \Psi(\sum_{j=1}^K\gamma_{dj})) \\
% &- \frac{1}{2}\log 2\pi \sigma^2 - \frac{1}{2\sigma^2}\big[y_d^2 -2y_d\eta^\top \tilde{\gamma}_d + \eta^\top \mathbb{E}[\theta_d\theta_d^\top]\eta \big ]\\
% &\sum_{n=1}^{N_d}\big[\sum_{i=1}^K\phi_{dni}(\Psi(\gamma_{di}) - \Psi(\sum_{j=1}^K\gamma_{dj})) + (\sum_{j=1}^V w_{dn}^j\varphi_j)(\sum_{i=1}^K\sum_{j=1}^V\phi_{dni}w_{dn}^j \log \beta_{ij}) \\
% & + (1- \sum_{j=1}^V w_{dn}^j\varphi_j )(\sum_{j=1}^V w_{dn}^j\log \pi_j) + (\sum_{j=1}^V w_{dn}^j\varphi_j)\log p + (1 - \sum_{j=1}^V w_{dn}^j\varphi_j)\log (1-p)\big]\\
% &-\log \Gamma(\sum_{j=1}^K\gamma_{dj}) + \sum_{i=1}^K \log \Gamma(\gamma_{di}) -  \sum_{i=1}^K(\gamma_{di} - 1)(\Psi(\gamma_{di} ) - \Psi(\sum_{j=1}^K\gamma_{dj}))\\
% & - \sum_{n=1}^{N_d}\sum_{i=1}^K \phi_{dni}\log\phi_{dni}\\
% &-\sum_{n=1}^{N_d} ((\sum_{j=1}^V w_{dn}^j\varphi_j)\log(\sum_{j=1}^V w_{dn}^j\varphi_j) + (1- \sum_{j=1}^V w_{dn}^j\varphi_j)\log(1- \sum_{j=1}^V w_{dn}^j\varphi_j))\Bigg\} \end{align*}

% where $\tilde{\gamma}_i := \mathbb{E}(\theta) = \frac{\gamma}{\gamma^\top 1}$ and $\mathbb{E}(\theta\theta^\top)$ is defined by:
% $\mathbb{E}(\theta\theta^\top)_{ij} = \frac{\tilde{\gamma}_i (\delta(i,j) - \tilde{\gamma}_j)}{\gamma_0 + 1} + \tilde{\gamma}_i\tilde{\gamma}_j$ and $\gamma_0 := \sum_{i=1}^K\gamma_i$. 
\subsection{Lower Bounds on the Log Likelihood}
\label{sec:lik_lb_deriv}
Remark that the likelihood for the words of one document can be written as follows:     

\[
p(\bw) = \int_{\theta}d\theta p(\theta|\alpha) \left\{\prod_{n=1}^N [p*p_{\beta}(w_n|\theta)+ (1-p)p_{\pi}(w_n)] \right\} 
\]

We would like to derive a lower bound to the joint log likelihood $p(y,\bw)$ of one document that resembles the prediction constrained log likelihood since they exhibit similar empirical behavior. Write $p(y,\bw)$ as $E_{\xi}[p(y|\bw,\xi)p(\bw|\xi)]$ and apply Jensen's inequality:

\[
\log p(y,\bw) \geq E_{\xi}[\log p(y|\bw,\xi)] + E_{\xi}[\log p(\bw|\xi)]
\]
Focusing on the second term we have:
\[
\log p(\bw|\xi) = \log \int_{\theta} d\theta p(\theta|\alpha) \prod_{n=1}^N p_{\beta}(w_n|\theta)^{\xi_n}p_{\pi}(w_n)^{1-\xi_n}
\]
Applying Jensen's inequality again to push the $\log$ further inside the integrals:
\[
\log p(\bw|\xi) \geq \int_{\theta}d\theta p(\theta|\alpha) \Bigg\{ \sum_{i=1}^N\xi_n \log p_{\beta}(w_n|\theta) + \sum_{n=1}^N(1-\xi_n)\log p_{\pi}(w_n) \Bigg\}
\]

Note that $\theta$ and $\xi$ are independent, so we have:

\[
\log p(y,\bw) \geq E[\log p(y|\bw,\xi)] + E\left[\sum_{i=1}^N\xi_n \log p_{\beta}(w_n|\theta) + \sum_{n=1}^N(1-\xi_n)\log p_{\pi}(w_n)\right]
\]

where the expectation is taken over the $\xi$ and $\theta$ priors. This gives the final bound:

\[
\log p(y,\bw) \geq E[\log p_{\beta}(y|W_1(\xi))] + pE[\log p_{\beta}(\bw|\theta)] + (1-p)\log p_{\pi}(\bw)
\]

We have used the substitution: $p(y|\bw,\xi) = p_{\beta}(y|W_1(\xi))$. Conditioning on $\xi$, $y$ is independent from the set of $w_n$ with $\xi_n = 0$, so we denote $W_1(\xi)$ as the set of $w_n$ with $\xi_n = 1$. It is also clear that $p(y|W_1(\xi), \xi) = p_\beta(y|W_1(\xi))$. By linearity of expectation, this bound can easily be extended to all documents.

Note that this bound is undefined on the constrained parameter space: $\beta^\top \pi = 0$; if $p\neq 0$ and $p\neq 1$. This is clear because $\log p_{\pi}(\bw)$ or $\log p_{\beta}(w_n|\theta)$ is undefined with probability 1. We can also see this directly, since $p(y,\bw|\xi)$ is non-zero for exactly one value of $\xi$ so $E[\log p(y,\bw|\xi)]$ is clearly undefined. We derive a tighter bound for this particular case as follows. Define $\xi^*(\pi, \beta, \bw)$ as the unique $\xi$ such that $p(\bw|\xi)$ is non-zero. We can write $p(y,\bw) = p(y,W | \xi^*(\pi, \beta, \bw)) p(\xi^*(\pi, \beta, \bw))$. For simplicity, I use the notation $\xi^*$ but keep in mind that it's value is determined by $\beta$, $\pi$ and $\bw$. Also remark that the posterior of $\xi$ is a point mass as $\xi^*$. If we repeat the analysis above we get the bound:
\[
\log p(y,\bw) \geq p_\beta(y| W_1(\xi^*)) + E\left[\sum_{n=1}^N\xi^*p_\beta(w_n|\theta)\right] + \sum_{n=1}^N(1-\xi^*)\log p_{\pi}(w_n) + p(\xi^*)
\]
which is to be optimized over $\beta$ and $\pi$. Note that the $p(\xi^*)$ term is necessary because of its dependence on $\beta$ and $\pi$.  
 Comparing this objective to our ELBO, we make a number of points. The true posterior is $\xi^*$ which would ordinarily require a combinatorial optimization to estimate; however we introduce the continuous variational approximation $\xi \sim Bern(\varphi)$. Note that the true posterior is a special case of our variational posterior (when $\varphi =1$ or $\varphi = 0$). Since the parameterization is differentiable, it allows us to estimate $\xi^*$ via gradient descent. Moreover, the parameterization encourages $\beta$ and $\pi$ to be disjoint without explicitly searching over the constrained space. Empirically, the estimated set of $\varphi$ are correct in simulations, and correct given the learned $\beta$ and $\pi$ on real data examples.

\subsection{Implementation details}
\label{sec:impl_deets}
In general, we treat $\alpha$ (the prior for the document topic distribution) as fixed (to a vector of ones). We tune pc-SLDA using \cite{hughes2017prediction}'s code base, which does a small grid search over relevant parameters. We tune sLDA and pf-sLDA using our own implementation and SGD. $\beta$ and $\pi$ are initialized with small, random (exponential) noise to break symmetry. We optimize using ADAM with initial step size 0.025.

We model real targets as coming from $N(\eta^\top \theta, \delta)$ and binary targets as coming from Bern$(\sigma(\eta^\top  \theta))$
\subsection{pf-sLDA likelihood and prediction constrained training.}
\label{sec:pf-sLDA_true_lik_to_pc}
The pf-sLDA marginal likelihood for one document and target can be written as:
\begin{align*}
    p(\bw, y) &= p(y | \bw) \int_\theta \sum_{\xi} p(\bw, \theta, \xi)\\ 
    &= p(y | \bw) \int_\theta p(\theta | \alpha) \prod_n\big\{ p * p_\beta(w_n | \theta, \xi_n = 1, \beta) + (1-p) * p_\pi(w_n | \xi_n = 0, \pi)\big\}
\end{align*}
where $n$ indexes over the words in the document. We see there still exist the $p(y|\bw)$ and $p * p_\beta(\bw)$ that are analagous to the prediction constrained objective, though the precise form is not as clear.

\subsection{Data set details}
\label{sec:dataset_deets}
\begin{itemize}
     \item Pang and Lee's movie review data set \citep{Pang2005SeeingSE}: There are 5006 documents. Each document represents a movie review, and the documents are stored as bag of words and split into 3754/626/626 for train/val/test. After removing stop words and words appearing in more than $50\% $ of the reviews or less than $10$ reviews, we get $|V| = 4596$. The target is an integer rating from $1$ (worst) to $10$ (best).
     \item Yelp business reviews \citep{Yelp:2019}: We use a subset of 10,000 documents from the Yelp 2019 Data set challenge . Each document represents a business review, and the documents are stored as bag of words and split into 7500/1250/1250 for train/val/test. After removing stop words and words appearing in more than $50\% $ of the reviews or less than $10$ reviews, we get $|V| = 4142$. The target is an integer star rating from $1$ to $5$.

     \item Electronic health records (EHR) data set of patients with Autism Spectrum Disorder (ASD), introduced in \cite{Masood2018APV}: There are 3804 documents. Each document represents the EHR of one patient, and the features are possible diagnoses. The documents are split into 3423/381 for train/val, with $|V| = 3600$. The target is a binary indicator of presence of epilepsy.
     \item Synthetic: We simulate a set of 5 data sets. Each data set is generated based on the pf-sLDA generative process with $\beta$ and $\pi$ random, each with mass on $50$ features. This means there are 100 total features, 50 relevant and 50 irrelevant. Other relevant parameters are $K=5$, $\alpha = \bold{1}$, and $p = 0.25$. We use these data sets to test the effectiveness and reliability of the feature filtering of pf-sLDA under the well-specified case when we have ground-truth.
 \end{itemize}

\subsection{Full Pang and Lee Movie Review Topics}

See Table below.

\label{sec:full_PL_topics}
  \begin{table*}[t!]
   \begin{center}
\begin{tabular}{|p{0.4cm}||p{4cm}|p{4.3cm}|p{4.6cm}|}
 \hline
     & sLDA & pc-sLDA, $\lambda=10$ & pf-sLDA, $p=0.10$ \\
 \hline
 1 & motion, way, love, performance, \textcolor{lasallegreen}{best}, picture, films, 
character, characters, life &
\textcolor{lasallegreen}{best}, little, time, \textcolor{lasallegreen}{good}, don, picture, year, rated, films
just & 
wars, \textcolor{lasallegreen}{emotionally}, allows, \textcolor{lasallegreen}{academy}, perspective, tragedy, today, 
\textcolor{lasallegreen}{important}, \textcolor{lasallegreen}{oscar}, \textcolor{lasallegreen}{powerful}   \\
\hline 
 $\eta_1$ &  7.801 & 8.287 & 10.253 \\
\hline \hline
 2&  kind, \textcolor{red}{poor}, \textcolor{lasallegreen}{enjoyable}, picture, \textcolor{lasallegreen}{excellent}, money, look, don
films, year
& little, just, good, doesn, life, way, films, character, time, characters
& \textcolor{jonquil}{complex}, study, \textcolor{lasallegreen}{emotions}, \textcolor{lasallegreen}{rare}, \textcolor{lasallegreen}{perfectly}, \textcolor{lasallegreen}{wonderful}, \textcolor{lasallegreen}{unique}, power, \textcolor{lasallegreen}{fascinating}, \textcolor{lasallegreen}{perfect} \\
\hline 
 $\eta_2$ &  5.994 & 8.127 & 8.792 \\
\hline \hline
 3 & running, subject, 20, character, message, characters, minutes
just, make, time & country, king, stone, dark, political, parker, mood, modern, dance, noir
& jokes, idea, wasn, \textcolor{jonquil}{silly}, \textcolor{red}{predictable}, \textcolor{jonquil}{acceptable}, \textcolor{red}{unfortunately}, \textcolor{red}{tries},  \textcolor{lasallegreen}{nice}, problem \\
\hline 
 $\eta_3$ &  5.735 & 4.407 & 6.258 \\
\hline \hline
 4 & \textcolor{jonquil}{acceptable}, language, teenagers, does, make, good, sex
, violence, rated, just & subscribe, room, jane, \textcolor{red}{disappointment}, michel, screening, primarily, reply, \textcolor{red}{frustrating}, plenty
& \textcolor{red}{tedious}, \textcolor{red}{poorly}, horror, \textcolor{red}{dull}, \textcolor{jonquil}{acceptable}, parody, \textcolor{red}{worse}, \textcolor{red}{ridiculous}, \textcolor{red}{supposed}, \textcolor{red}{bad} \\
\hline 
 $\eta_4$ &  5.064 & 3.440 & 2.657 \\
\hline \hline
5 & plot, time, \textcolor{red}{bad},  \textcolor{lasallegreen}{funny},  \textcolor{lasallegreen}{good},  \textcolor{lasallegreen}{humor}, little, isn, action & script, year, little,  \textcolor{lasallegreen}{good}, don, look, rated, picture, just, films 
& suppose, \textcolor{red}{lame}, \textcolor{red}{annoying}, \textcolor{red}{attempts}, \textcolor{red}{failed}, \textcolor{red}{attempt}, \textcolor{red}{boring}, \textcolor{red}{awful}, \textcolor{red}{dumb}, \textcolor{red}{flat}, 
comedy
\\
\hline 
 $\eta_5$ &  3.516 & 2.802 & 0.135 \\
\hline \hline
\end{tabular}
\end{center}
\caption{ We list the top 10 most likely words for each topic for the models specified. The topics are organized from highest to lowest with respect to it's corresponding coefficient for the supervised task. In the context of movie reviews, positive-sentiment words are listed in green, negative-sentiment words are listed in red, and sometimes positive, sometimes negative or neutral, but sentiment-related words are listed in yellow. Non-sentiment related words are listed in black.}
\label{tab:Pang_lee_full_topics}
\end{table*}

\subsection{Coherence details}
\label{sec:coherence_deets}
We calculate coherence for each topic by taking the top $N$ most likely words for the topic, calculating the pointwise mutual information for each possible pair, and averaging. These terms are defined below.
\begin{align*}
    \text{coherence} &= \frac{1}{N (N-1)} \sum_{w_i,w_j \in \text{TopN}} \text{pmi}(w_i, w_j) \\ 
    \text{pmi}(w_i, w_j) &= \log \frac {p(w_i) p(w_j)}{p(w_i, w_j)} \\ 
    p(w_i) &= \frac{\sum_d I(w_i \in \text{doc d})} M \\
    p(w_i, w_j) &= \frac{\sum_d I(w_i \text{ and } w_j \in \text{doc d})} M
\end{align*} 
where $M$ is the total number of documents and $N$ is the number of top words in a topic. The final coherence we report for a model is the average of all the topic coherences.

The numbers in the paper are reported with $N=50$, but we found that the general trends (i.e. pf-sLDA topics most coherent) were consistent across several values of $N$. For example, for the Pang and Lee Movie Review Dataset, we have the following coherences:

 \begin{table*}[hbt]
   \begin{center}
\begin{tabular}{|p{2cm}||p{1cm}|p{1cm}|p{1cm}|p{1cm}|p{1cm}|}
 \hline
    N & 10 & 20 & 30 & 40 & 50 \\
    \hline
    pc-sLDA & 0.42 & 0.44 & 0.44 & 0.46 & 0.49 \\
    \hline
    pf-sLDA & 1.52 & 1.69 & 1.76 & 1.92 & 1.99 \\
 \hline
 \end{tabular}
 \end{center}
 \end{table*}

\subsection{Coordinate Ascent Updates}
\label{sec:caupdates}
\subsubsection{E Step}

Denote $L_\varphi$ the terms of the ELBO that depend on $\varphi$ and similarly for other parameters.

\subsubsection*{$\varphi$ update}

\begin{align*}
    L_\varphi = &\sum_n \left(\sum_v w_{nv} \varphi_v \right) \left(\sum_k \sum_v \phi_{nk} w_{nv} \log \beta_{kv}\right) + \left(1 - \left(\sum_v w_{nv} \varphi_v \right)\right) \left(\sum_v w_{nv} \log \pi_v \right) \\
    &+  \sum_n \left(\sum_v w_{nv} \varphi_v \right) \log p + \left(1 - \left(\sum_v w_{nv} \varphi_v \right)\right)\log (1-p) \\
    &- \sum_n \left(\sum_v w_{nv} \varphi_v \right) \log \left(\sum_v w_{nv} \varphi_v \right) + \left(1 - \left(\sum_v w_{nv} \varphi_v \right)\right) \log \left(1 - \left(\sum_v w_{nv} \varphi_v \right)\right) \\
\end{align*}

\begin{align*}
    \frac{\partial}{\partial \varphi_j}L_\varphi = &\sum_n w_{nj} \left(\sum_k \sum_v \phi_{nk} w_{nv} \log \beta_{kv}\right) - w_{nj}\left(\sum_v w_{nv} \log \pi_v \right) \\
    &+ \sum_n w_{nj} (\log p - \log (1 - p)) \\
    &- \sum_n w_{nj} \left(\log \left(\sum_v w_{nv} \varphi_v \right) - \log \left(1 - \left(\sum_v w_{nv} \varphi_v \right)\right)\right) \\
    =  &\sum_n w_{nj} \left(\sum_k \phi_{nk} \log \beta_{kj} - \log \pi_j\right) \\
    &+ \sum_n w_{nj} (\log p - \log (1 - p)) \\
    &- \sum_n w_{nj} \left(\log \varphi_j - \log \left(1 - \varphi_j\right)\right)
\end{align*}

Let $$\Omega_j = \sum_n w_{nj} \left(\sum_k \phi_{nk} \log \beta_{kj} - \log \pi_j + \log p - \log (1 - p)\right)$$. 

Setting the gradient to 0 and solving, we get the following update:
$$ \varphi_v \leftarrow \frac{\exp(\Omega_j/\sum_n w_{nj})}{1+ \exp(\Omega_j/\sum_n w_{nj})}$$

This update makes intuitive sense. It looks like the ratio between the probability of $\beta$ and $\pi$ explaining the word, weighted by $p$, normalized, sigmoided to make it a valid probability.

% \begin{align*}
% \frac{\partial}{\partial \varphi_v}L = & \sum_{d=1}^M\Bigg\{ \sum_{n=1}^{N_d}\big[(\sum_{i=1}^K\sum_{j=1}^V\phi_{dni}w_{dn}^j \log \beta_{ij})w_{dn}^v - (\sum_{j=1}^V w_{dn}^j\log \pi_j)w_{dn}^v +(\log p )w_{dn}^v - \log(1-p)w_{dn}^v\big] \\
% & - W_{dv}(\log \varphi_v - \log(1-\varphi_v)) \Bigg\}
% \end{align*}
% where $W_d$ is the BoW vector for document $d$. Denoting the first line as $\Theta_v$ and $\Omega_v := \sum_d W_{dv}$, this gives:

% \[
% \varphi_v \leftarrow \frac{\exp(\Theta_v/\Omega_v)}{1+ \exp(\Theta_v/\Omega_v)}
% \]

\subsubsection*{$\phi$ update}
Assuming $y$ depends on $\theta$ for now.
\begin{align*}
     L_\phi = &\sum_n \sum_k \phi_{nk} \left(\Psi(\gamma_k) - \Psi\left(\sum_{j=1}^K\gamma_{j}\right)\right) \\
      &+ \sum_n \left(\sum_v w_{nv} \varphi_v \right) \left(\sum_k \sum_v \phi_{nk} w_{nv} \log \beta_{kv}\right) \\
      &- \sum_n \sum_k \phi_{nk} \log \phi_{nk} \\
\end{align*}

Let $w_n = v$ and using Lagrange Multipliers with the constraint $\sum_k \phi_{nk} = 1$:
\begin{align*}
    \frac{\partial}{\partial\phi_{nk}}L_\phi =\Psi(\gamma_k) - \Psi\left(\sum_{j=1}^K\gamma_{j}\right) + \varphi_v\log \beta_{kv} - 1- \log\phi_{nk} + \lambda 
\end{align*}

This gives the update:
$$
\phi_{nk}\propto \beta_{kv}^{\varphi_v}\exp\left(\Psi(\gamma_{k}) - \Psi\left(\sum_{j=1}^K\gamma_{j}\right)\right)
$$

This is the same update as LDA, with $\beta$ weighted by $\varphi$.

% We can write (letting $v$ denote the word associated with word $n$):

% \[
% L(\phi_{dnk}) = \phi_{dnk}(\Psi(\gamma_{dk}) - \Psi(\sum_{j=1}^K\gamma_{dj})) + \varphi_v (\phi_{dnk}\log \beta_{kv}) - \phi_{dnk}\log\phi_{dnk} +\lambda (\sum_{j=1}^K\phi_{dnj} -1)
% \]

% \[
% \frac{\partial}{\partial\phi_{dnk}}L = \Psi(\gamma_{dk}) - \Psi(\sum_{j=1}^K\gamma_{dj}) + \varphi_v\log \beta_{kv} - 1- \log\phi_{dnk} + \lambda 
% \]

% This gives the update:
% \[
% \phi_{dnk}\propto \beta_{kv}^{\varphi_v}\exp\{\Psi(\gamma_{dk}) - \Psi(\sum_{j=1}^K\gamma_{dj})\}
% \]

\subsubsection*{$\gamma$ Update}
\begin{align*}
    L_\gamma = &\sum_k(\alpha_k - 1)\left(\Psi(\gamma_k) - \Psi\left(\sum_{j=1}^K\gamma_{j}\right)\right) \\
    &+ \sum_n \sum_k \phi_{nk}\left(\Psi(\gamma_k) - \Psi\left(\sum_{j=1}^K\gamma_{j}\right)\right) \\
    &- \log \Gamma\left(\sum_k \gamma_k\right) - \sum_k \log \Gamma(\gamma_k) + \sum_k(\gamma_k - 1)\left(\Psi(\gamma_k) - \Psi\left(\sum_{j=1}^K\gamma_{j}\right)\right) \\
    &- \frac{1}{2\sigma^2}\left(y^2 -2y\eta^\top E_q[\theta] + \eta^\top E_q[\theta \theta^\top]\eta \right)
\end{align*}

The gradient of $\gamma$ consists of two components; the first (for the first three lines) is the same as in LDA: 
$$
\frac {\partial} {\partial \gamma_i} L1_\gamma = \Psi'(\gamma_{i})\left(\alpha_i + \sum_n \phi_{ni} - \gamma_{i}\right) - \Psi'\left(\sum_k \gamma_{k}\right)\sum_k\left(\alpha_k + \sum_n\phi_{nk} - \gamma_{k}\right)
$$

The second term is as follows:

\begin{align*}
\frac{\partial L2_\gamma}{\partial \gamma_i} = &-\frac{1}{2\sigma^2}[-2y_d(\frac{\eta_k}{\gamma_0} - \frac{\sum_{j=1}^K\eta_i\gamma_{dj}}{\gamma_0^2}) + 2\sum_{i\neq k} \eta_i\eta_k\frac{-\gamma_{di} \gamma_0^2(\gamma_0 + 1) + \gamma_{dk}\gamma_{di}(2\gamma_0(\gamma_0 + 1) + \gamma_0^2)}{\gamma_0^4(\gamma_0 + 1)^2}\\
& + \eta_k^2\frac{(1-2\gamma_{dk}) \gamma_0^2(\gamma_0 + 1) + \gamma_{dk}(\gamma_{dk} -1)(2\gamma_0(\gamma_0 + 1) + \gamma_0^2)}{\gamma_0^4(\gamma_0 + 1)^2} + 2\sum_{i\neq k}\eta_k\eta_i \frac{\gamma_i\gamma_0^2 - 2\gamma_0\gamma_i\gamma_k}{\gamma_0^4}\\
&+ \eta_k^2\frac{2\gamma_k\gamma_0^2 - 2\gamma_0\gamma_k^2}{\gamma_0^4}]
\end{align*}

\subsubsection{M Step}
\subsubsection*{$\pi$ Update}
\begin{align*}
    L_\pi = \sum_n \left(1 - \left(\sum_v w_{nv} \varphi_v \right)\right) \left(\sum_v w_{nv} \log \pi_v \right)
\end{align*}
Taking the derivative and using lagrange multipliers with the constraint $\sum_v \pi_v = 1$, we get:
\begin{align*}
    \frac{\partial}{\partial \pi_j}L_\pi &= \sum_n \left( 1 - \sum_v w_{nv} \varphi_v\right) \frac {w_{nj}} {\pi_j} + \lambda \\
    &= \sum_n \left( 1 - \varphi_j\right) \frac {w_{nj}} {\pi_j} + \lambda
\end{align*}

This gives the update: 
$$ \pi_j \propto (1-\varphi_j)\sum_n w_{nj} $$

This is an intuitive update - it is the proportion the word appears, weighted by $\varphi$.

% \begin{align*}
% \frac{\partial}{\partial \pi_v}L &= \frac{\partial}{\partial \pi_v} \sum_{d=1}^M\sum_{n=1}^{N_d} (1- \sum_{j=1}^Vw_{dn}^j\varphi_j)(\sum_{j=1}^V w_{dn}^j \log \pi_j) + \lambda(\sum_{j=1}^V \pi_j - 1)\\ & = \sum_{d=1}^M\sum_{n=1}^{N_d} (1- \varphi_v)(\frac{w_{dn}^v}{\pi_v}) + \lambda \\& \rightarrow \pi_v \propto (1-\varphi_v)\sum\sum w_{dn}^v
% \end{align*}

\subsubsection*{$\beta$ update}
\begin{align*}
    L_\beta =  \sum_n \left(\sum_v w_{nv} \varphi_v \right) \left(\sum_k \sum_v \phi_{nk} w_{nv} \log \beta_{kv}\right)
\end{align*}

Adding a Lagrange multiplier for constraint $\sum_v \beta_{kv} = 1$  $\forall k$ and then taking gradient:
\begin{align*}
    \frac{\partial}{\partial \beta_{kj} }L_\beta = \sum_n \frac{\varphi_j \phi_{nk} w_{nj}}{\beta_{kv}} + \lambda
\end{align*}

This gives the update:
$$ 
\beta_{kj}\propto \sum_n \varphi_j \phi_{nk} w_{nj}
$$

This is the same update as sLDA, but weighted by $\varphi$.

\subsection*{$\alpha$ update} This should be exactly the same as in LDA and sLDA via Newton-Raphson. 

\subsubsection*{$\eta$ update}
Same form as sLDA updates, but replace $E[\bar{Z}]$ and $E[\bar{Z} \bar{Z}^\top]$ with $E[\theta]$ and $E[\theta \theta^\top]$. 
\subsubsection*{$\delta$ update}
Same form as sLDA updates, but replace $E[\bar{Z}]$ and $E[\bar{Z} \bar{Z}^\top]$ with $E[\theta]$ and $E[\theta \theta^\top]$.

\end{document}